\documentclass{article}
\usepackage{ijcai20}
\usepackage{xcolor}
\usepackage{todonotes} 
\usepackage{natbib}
\usepackage{times}
\usepackage{soul}
\usepackage{url}
\usepackage[hidelinks]{hyperref}
\usepackage[utf8]{inputenc}
\usepackage[small]{caption}
\usepackage{graphicx} 
\usepackage{subfigure} 
\graphicspath{{figures/}}
\usepackage{amsmath}
\usepackage{amsthm}
\usepackage{booktabs}
\usepackage{algorithm}
\usepackage{algorithmic}
\usepackage{stackengine,xcolor}

\ifx\proof\undefined
\newenvironment{proof}{\par\noindent{\bf Proof\ }}{\hfill\BlackBox\\[2mm]}
\fi

\ifx\theorem\undefined
\newtheorem{theorem}{Theorem}
\fi
\ifx\example\undefined
\newtheorem{example}{Example}
\fi
\ifx\lemma\undefined
\newtheorem{lemma}[theorem]{Lemma}
\fi

\ifx\corollary\undefined
\newtheorem{corollary}[theorem]{Corollary}
\fi

\ifx\assumption\undefined
\newtheorem{assumption}{Assumption}
\fi

\ifx\definition\undefined
\newtheorem{definition}{Definition}
\fi

\ifx\proposition\undefined
\newtheorem{proposition}[theorem]{Proposition}
\fi

\ifx\remark\undefined
\newtheorem{remark}{Remark}
\fi

\ifx\conjecture\undefined
\newtheorem{conjecture}[theorem]{Conjecture}
\fi

\ifx\factoid\undefined
\newtheorem{factoid}[theorem]{Fact}
\fi

\ifx\axiom\undefined
\newtheorem{axiom}[theorem]{Axiom}
\fi

\newcommand{\RN}[1]{%
	\textup{\lowercase\expandafter{\it \romannumeral#1}}%
}

\def\ReLU{\textsf{ReLU}} 
\def\Softmax{\textsf{Softmax}} 
\def\att{\textsf{Att}}
\def\LeakyReLU{\textsf{LeakyReLU}}
\def\MMD{\textsf{MMD}}

\input{Definitions}

\title{Graph Neural Networks with Composite Kernels}

\author{
Yufan Zhou$^1$\footnote{Contact Author}\and
Jiayi Xian$^1$\and
Changyou Chen$^{1}$\And
Jinhui Xu$^1$\\
\affiliations
$^1$University at Buffalo\\
\emails
\{yufanzho, jxian, changyou, jinhui\}@buffalo.edu,
}

\begin{document}

\maketitle

\begin{abstract}
  Learning on graph structured data has drawn increasing interest in recent years. Frameworks like Graph Convolutional Networks (GCNs) have demonstrated their ability to capture structural information and obtain good performance in various tasks. In these frameworks, node aggregation schemes are typically used to capture structural information: a node's feature vector is recursively computed by aggregating features of its neighboring nodes. However, most of aggregation schemes treat all connections in a graph equally, ignoring node feature similarities. In this paper, we re-interpret node aggregation from the perspective of kernel weighting, and present a framework to consider feature similarity in an aggregation scheme. Specifically, we show that normalized adjacency matrix is equivalent to a neighbor-based kernel matrix in a Krein Space. We then propose feature aggregation as the composition of the original neighbor-based kernel and a learnable kernel to encode feature similarities in a feature space. We further show how the proposed method can be extended to Graph Attention Network (GAT). Experimental results demonstrate better performance of our proposed framework in several real-world applications.
\end{abstract}

\section{Introduction}
Learning on graph structured data has drawn a considerable amount of interest in recent years, due to the fact that various real world data are in the form of graphs, {\it e.g.}, knowledge graphs \citep{nickel2015review}, citation networks \citep{kipf2016semisupervised}, social networks \citep{chen2018fastgcn}, molecular graph \citep{wu2018moleculenet}, {\it etc}. Among different methods, Graph Convolutional Network (GCN) \citep{kipf2016semisupervised} and its variants have achieved the state-of-the-art performance in many real-world tasks with graph structured data. By recursively aggregating features from its neighboring nodes, a node's high-level representations (containing some structural information) can be learned effectively. Roughly speaking, there are two major ways to design an aggregation scheme: through a pre-defined scheme and a learnable scheme. Specifically, GCN uses elements in the normalized adjacency matrix as aggregation coefficients; whereas Graph Attention Network (GAT) \citep{velikovi2017graph} uses learnable attentions as the aggregation coefficients. 

Despite numerous successes achieved by the aforementioned schemes, they often lack sufficient flexibility to encode node features. 
For a given node in the graph, it is quite natural to expect it to 
pay more attentions to those neighboring nodes that are closer to it in the feature space, rather than to treat all of them  identically. Thus, it is desirable for the aggregation scheme to be able to differentiate nodes based on their features. 
In this paper, we achieve this goal by proposing a framework that takes the feature similarity into consideration. Our main idea is to reinterpret the standard aggregation in GCN as a kernel aggregation scheme, and then generalize the aggregation by considering composite kernels, where feature similarity can be effectively encoded by a newly proposed learnable kernel. Our framework is also flexible enough and can be extended to the GAT framework.  Our contributions can be summarized as follows. 
\begin{itemize}
    \item We interpret the aggregation operation in standard GCN from a perspective of kernel aggregation, and show that normalized adjacency matrix is actually kernel matrix associated with Reproducing Kernel Krein Space. 
    \item We propose a composite kernel construction framework to define a feature sensitive aggregation. Our proposed framework takes feature similarity into consideration and can be easily applied to different models like GCN and GAT.
    \item We provide a new way to learn the characteristic positive semi-definite kernel to encode feature similarities. Some theoretical analysis and explanations are also provided.
    \item Our proposed framework achieves better performance compared to the standard GCN, GAT and some other related models on several tasks.
\end{itemize}

\paragraph{Comparisons with existing work}
Attention mechanism is related to our method as it is also
a learnable aggregation scheme. Both the attention coefficients and our kernel function are learned based on features of the nodes and their neighbors. In GAT, a single layer neural network is used to learn the attention coefficient. In our model, a more flexible way is proposed, which can be interpreted as a composition kernel learning problem where label information can also be effectively encoded. Furthermore, our framework can be applied to GAT, and achieves better performance as shown
in Section \ref{sec:exp}.

There are some other work applying kernel methods to graph data. For example, \citet{tian2019rethinking} proposes a framework to learn kernel functions on a graph, which provides an understanding of GCN from a kernel point of view. Similar to our work, their framework proposes a way to learn a kernel driven by node label information. The major difference from ours is that they learn a feature mapping for a kernel function, which still uses the normalized adjacency matrix (and its powers) in the aggregation scheme. In other words, they still treat all the neighbors identically. Experimental comparisons are provided in Section \ref{sec:exp}.

\section{Preliminary}\label{sec:pre}

We begin by introducing some basic concepts about graph, kernel and related models. An undirected graph is represented as $\mathcal{G}=(\mathcal{V}, \mathcal{E})$, where $\vb_i \in \mathcal{V}$ denotes a node in the graph and $e_{i,j} \in \mathcal{E}$ denotes the edge of a graph connecting nodes $\vb_i$ and $\vb_j$. We use $\mathcal{N}_i$ to denote the set containing $\vb_i$ and all its neighbors.

Every node in the graph is associated with a feature vector and a label. We use $\mathbf{X} \in R^{n\times d}$ to denote the node features and $\mathbf{Y} \in R^{n\times c}$ to denote the node labels, where $c$ is the class number, $d$ is the feature dimensionality and $n$ is the number of nodes in graph. The feature and label of a given node $\vb_i$ are denoted as $\xb_i$ and $\yb_i$, which are the $i^{th}$ row of $\mathbf{X}$ and $\mathbf{Y}$, respectively. 

We use $\mathbf{H}^l \in R^{n \times d_l}$ to denote the outputs of the $l^{th}$ hidden layer in GNNs, $W^l$ to denote the weight matrix in the $(l+1)^{th}$ layer, where $d_l$ denotes the dimensionality of the outputs of the $l^{th}$ layer. $\hb _i^l$ is the $i^{th}$ row of $\mathbf{H}^l$ with input $\vb_i$. For consistency, we denote $\mathbf{H}^0 = \mathbf{X}$ and $\hb_i^0 = \xb_i$.

The adjacency matrix $\mathbf{A} \in R^{n\times n}$ is a symmetric (typically sparse) matrix used to encode edge information in a graph. Elements $a_{i,j}=1$ if there is an edge between $\vb_i$ and $\vb_j$, and $a_{i,j} = 0$ otherwise. A degree matrix $\mathbf{D}=\text{diag}(d_1, d_2...d_n)$ is a diagonal matrix with node degrees as its diagonal elements, {\it i.e.}, $d_j = \sum_j a_{i,j}$. 

In the following, we describe GCN and GAT in the setting of node classification, where the softmax function is used as the activation function of the last layer, and the cross entropy is used as the loss function.

\subsection{Graph Convolutional Networks}

\citet{kipf2016semisupervised} extended the idea of Convolutional Neural Network (CNN) to graph data and proposed Graph Convolutional Networks (GCNs). Specifically, let $\Tilde{\mathbf{A}} = \mathbf{A} + \mathbf{I}$ with $\mathbf{I}$ being the identity matrix. $\Tilde{\mathbf{A}}$ can be considered as the adjacency matrix with self-connection. The corresponding degree matrix $\Tilde{\mathbf{D}}$ is a diagonal matrix with diagonal elements $\Tilde{d}_j = \sum _j \Tilde{a}_{i,j}= d_j + 1$. Similar to CNNs, an $L$-layer GCN defines a sequence of transformations, {\it i.e.}, the output of the $(l+1)^{th}$ layer is calculated as: 
$\hb_i^{l+1} = \sigma(\sum _{j \in \mathcal{N}_i}\hat{a} _{i,j} \hb_j^l \mathbf{W}^l)$, 
where $\sigma$ denotes the activation function (softmax for the last layer and ReLU for other layers), $\hat{a}_{i,j}$ is the $(i,j)$-th element in $\hat{\mathbf{A}} \triangleq \Tilde{\mathbf{D}}^{-1/2} \Tilde{\mathbf{A}}\Tilde{\mathbf{D}}^{-1/2}$.

\paragraph{Node aggregation}
In GCN, the matrix $\hat{A}$ is called an aggregation operator since its role is to aggregate representations of neighboring nodes to calculate representation for a particular node. The way it works is to simply take the product of $\hat{A}$ and the node representation matrix. For example, in a 2-layer GCN for node classification, the output can be written as 
\begin{align} \label{eq:gcn_graph}
    f(\mathbf{X}) = \Softmax(\hat{\mathbf{A}}\ReLU(\hat{\mathbf{A}}\mathbf{X}\mathbf{W}^0)\mathbf{W}^1),
\end{align}
where $\mathbf{W}^0$ and $\mathbf{W}^1$ denote the weight matrices in the first and second layers. Softmax and ReLU functions are applied pointwise. For a given node, the 2-layer GCN can also be represented as:
{\begin{align*}
    f(\vb_i) = \Softmax(\sum _{j \in \mathcal{N}_i}\hat{a}_{i,j} \ReLU(\sum _{j \in \mathcal{N}_i}\hat{a}_{i,j} \xb_j \mathbf{W}^0) \mathbf{W}^1).
\end{align*}}
When considering $\xb_j \mathbf{W}^0$ as a transformed representation for node $j$, the term $\sum _{j \in \mathcal{N}_i}\hat{a}_{i,j} \xb_j \mathbf{W}^0$ essentially aggregates neighbor node representations to form a new representation for node $i$.

It is worth noting that the aggregation scheme in GCN uses the normalized adjacency matrix $\hat{\mathbf{A}}$, which is not learnable and treats all the neighbors identically in a graph.

\subsection{Graph Attention Networks}
\citet{velikovi2017graph} extends GCNs to Graph Attention Networks (GATs), which uses an attention mechanism for aggregation. Specifically, for two hidden representations $\hb_i^l$ and $\hb_j^l$, an attention mechanism $\att: R^{d_l} \times R^{d_l} \rightarrow R$ computes the attention coefficients $t_{i,j} = \att(\hb_i^l\mathbf{W}^l, \hb_j^l\mathbf{W}^l)$, which calculates an importance score of node $\vb_j$ to $\vb_i$, in contrast to using $\hat{a}_{i,j}$ as in standard GCNs. 

In practice, the attention coefficients of a given node is often normalized across all its neighbors using the softmax function. In addition, $\att$ is often implemented as a single-layer neural network with Leaky ReLU activation function. The resulting formula for the attention mechanism is:
{\small\begin{align}\label{eq:att_computation}
    t_{i,j} =\dfrac{\exp(\LeakyReLU([ \hb_i^l \mathbf{W}^l \Vert  \hb_j^l \mathbf{W}^l]^{\intercal} \thetab ))}
                    {\sum _{k \in \mathcal{N}_i} \exp(\LeakyReLU([ \hb_i^l \mathbf{W}^l \Vert  \hb_k^l \mathbf{W}^l]^{\intercal} \thetab ))))},
\end{align}}
where ``$\Vert$'' denotes the vector concatenation operator, and $\thetab$ the learnable parameter vector for an attention mechanism.


In a multi-head attention mechanism, a hidden layer has $M$ independent attentions, resulting in an aggregation scheme formulated as:
\begin{align}\label{eq:multi_att}
    \hb_i^{(l+1)} = \Vert _{m=1}^M \ReLU(\sum _{j \in \mathcal{N}_i}t^{l,m}_{i,j} \hb_j^l \mathbf{W}^l),
\end{align}
where $t^{l,m}_{i,j}$ denotes the attention coefficients from the $m^{th}$ attention head in the $(l+1)^{th}$ layer, and $\Vert _{m=1}^M$ denotes the concatenation of $M$ resulting output vectors. 
For a classification task, the last layer of GAT is typically defined as:
\begin{align}\label{eq:last_multi_att}
        \hb_i^{L} = \Softmax(\sum _{m=1}^M \sum _{j \in \mathcal{N}_i}t^{L-1,m}_{i,j} \hb_j^{L-1} \mathbf{W}^{L-1})~.
\end{align}

\subsection{Kernel Function}
A kernel $k: R^d \times R^d \rightarrow R$ is a function with two variables. Given a set of data $\{\xb_1, \xb_2...\xb_n\}$ with $\xb_i \in R^d$, a kernel matrix $\mathcal{K}$ is the matrix whose elements are computed as $k_{i,j} = k(\xb_i, \xb_j)$. The merit of kernel is that it can describe nonlinear relations between data pairs by calculating their similarity in an infinite dimensional space. 

\paragraph{Positive semi-definite kernels} 
Positive semi-definite kernels are kernel functions whose matrices are positive semi-definite. They can be seen as inner products in Hilbert space, thus can be used to describe feature similarity or difference. Each positive semi-definite kernel is associated with a Reproducing Kernel Hilbert Space (RKHS).
\begin{definition}
A Hilbert space $(\mathcal{H}, \langle \cdot, \cdot \rangle_{\mathcal{H}})$ is a RKHS if its evaluation functional is continuous, where the evaluation functional at $x$ is defined as $T_x: \mathcal{H} \rightarrow R$ and satisfies $T_x (f) = f(x), f\in \mathcal{H}$.
\end{definition}

Positive semi-definite kernels can be used to calculate similarity of two probabilistic measures, for example, via the Maximum Mean Discrepancy (MMD) \citep{gretton2012kernel}, defined as:
\begin{align}
    \MMD^2(\mathbb{P}, \mathbb{Q})=&\mathbb{E}_{\xb_i, \xb_j \sim \mathbb{P}}(k(\xb_i, \xb_j)) + \mathbb{E}_{\yb_i, \yb_j \sim \mathbb{Q}}(k(\yb_i, \yb_j)) \nonumber \\
    &- 2\mathbb{E}_{\xb_i \sim \mathbb{P}, \yb_j \sim \mathbb{Q}}(k(\xb_i, \yb_j)).
\end{align}

\paragraph{Indefinite kernels}
These is also some research on indefinite kernels \citep{ong2004learning}, {\it i.e.}, kernels whose kernel matrices may not be positive semi-definite. Indefinite kernels are interesting because in kernel learning \citep{ong2003machine}, some solutions may not be positive semi-definite functions, {\it e.g.}, linear combination of positive semi-definite kernels. In contrast to positive semi-definite kernels, indefinite kernels can not be described in a RKHS. However, they can be defined as inner products in Reproducing Kernel Krein Spaces (RKKS).
\begin{definition}\citep{ong2004learning}
An inner product space is a Krein Space $(\mathcal{J}, \langle \cdot, \cdot \rangle_{\mathcal{J}})$, if there exist two Hilbert Spaces $\mathcal{H_+}$, $\mathcal{H_-}$ spanning $\mathcal{J}$ such that:
\begin{itemize}
    \item All $f \in \mathcal{J}$ can be decomposed into $f = f_+ + f_-$, where $f_+\in \mathcal{H_+}$ and $f_- \in \mathcal{H_-}$;
    \item $\forall f, g \in \mathcal{J}$, $\langle f, g\rangle_{\mathcal{J}} = \langle f_+, g_+\rangle_{\mathcal{H_+}} - \langle f_-, g_-\rangle_{\mathcal{H_-}}$.
\end{itemize}
\end{definition}

\begin{definition}\citep{ong2004learning}
A Krein space $(\mathcal{J}, \langle \cdot, \cdot \rangle_{\mathcal{J}})$ is an RKKS if its evaluation functional is continuous, where the evaluation functional at $x$ is defined as $T_x: \mathcal{J} \rightarrow R$ and satisfies $T_x (f) = f(x), \forall f\in \mathcal{J}$.
\end{definition}

\begin{lemma}\citep{ong2004learning}
If $\mathcal{H_+}$ and $\mathcal{H_-}$ are RKHS with kernels $k_+$ and $k_-$, then the resulting RKKS $\mathcal{J}$ has a unique indefinite kernel $k = k_+ - k_-$.
\end{lemma}
We can see that indefinite kernel can be decomposed into two positive semi-definite kernels. Note that both positive semi-definite kernels and indefinite kernels are symmetric, {\it i.e.}, $k(\xb_i, \xb_j) = k(\xb_j, \xb_i)$.

\section{The Proposed Method}
We first re-interpret GCNs from a kernel aggregation point of view, and then generalize the kernel aggregation to composite kernel aggregation so that it can incorporate node features. We further describe how to extend our framework to the class of GATs, followed by a detailed learning process. For illustration purpose, Figure \ref{fig:graph_example} shows the general idea of the proposed method, with details provided below.

\begin{figure}[t!]
    \centering
    \includegraphics[width=0.99\linewidth]{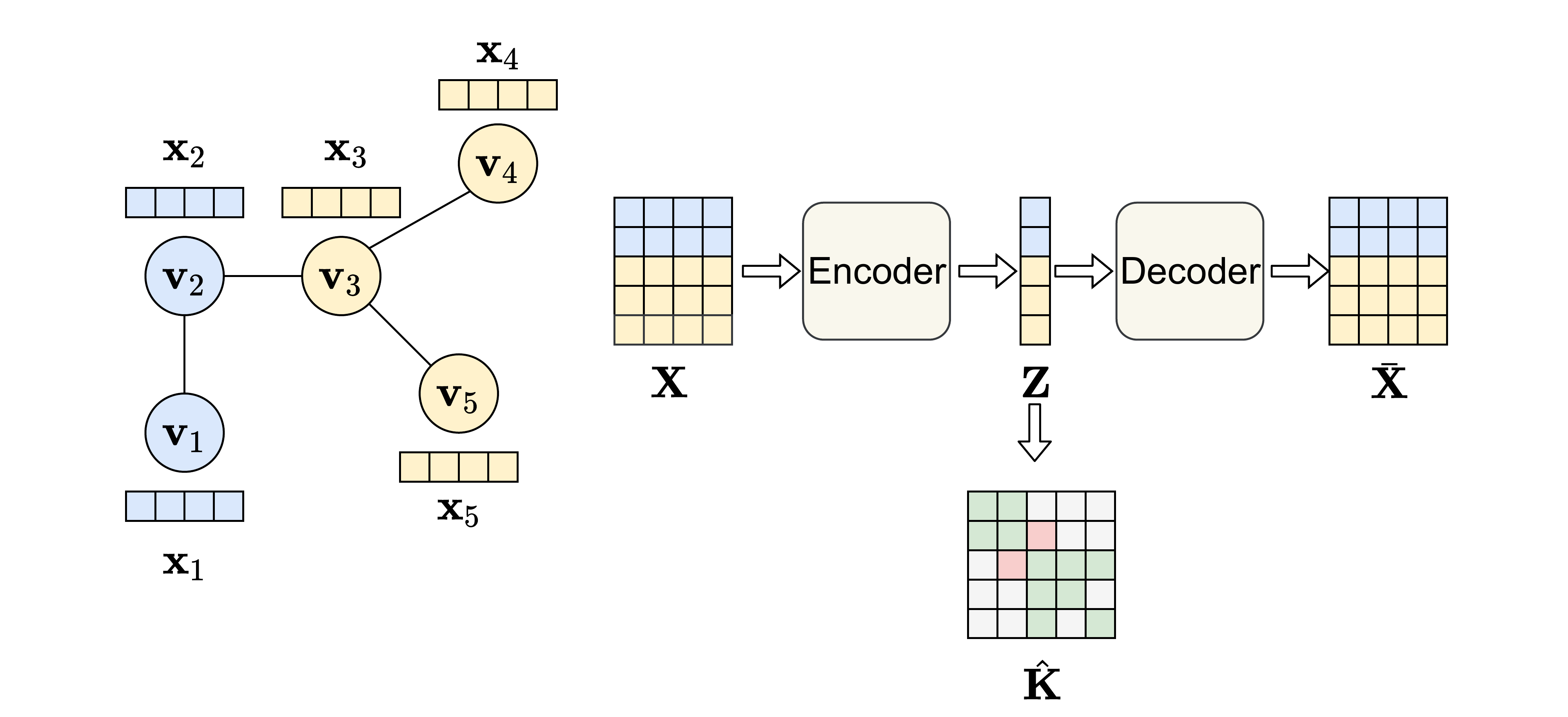}
    \vspace{-0.1in}
    \caption{Illustration of the proposed method. The example graph consists of 5 nodes and 4 edges. The nodes belongs to 2 classes (denoted by different colors) and every node has a 4-dimensional feature vector. Matrix $\hat{\mathbf{K}}$ is shown in the figure, with its non-zero elements colored in green or red, where green color means that the element's value should be maximized and red color means that the element's value should be minimized.}
    \label{fig:graph_example}
    \vspace{-0.1in}
\end{figure}

\subsection{GCNs as Kernel Aggregation}
As described in Section \ref{sec:pre}, a GCN uses a normalized adjacency matrix in its aggregation scheme. 
We explain in the following that this aggregation scheme is equivalent to aggregate with kernel weighting, {\it i.e.}, the elements $\hat{a}_{i,j}$ can be interpreted as elements of a kernel matrix.
\begin{proposition}\label{pro:A_hat_indefinite}
    Normalized adjacency matrix $\hat{\mathbf{A}}$ is a valid kernel matrix of an indefinite kernel.
\end{proposition}
\begin{proof}
The proposition can be proved by showing that $\mathbf{I}$ and $\mathbf{I} - \hat{\mathbf{A}}$ are positive semi-definite matrices \citep{ong2003machine}, where $\mathbf{I}$ is the identity matrix. For any $\bb \in R^d$, we have:
\footnotesize{
    \begin{align*}
            \bb^{\intercal}(\mathbf{I} - \hat{\mathbf{A}}) \bb &= \sum_i b_i^2 - \sum_{i,j} \dfrac{\hat{a}_{i,j}b_i b_j}{\sqrt{(d_i +1)(d_j + 1)}}\\
            &=\dfrac{1}{2}\sum_{i,j}\hat{a}_{i,j}(\dfrac{b_i}{\sqrt{d_i +1}} - \dfrac{b_j}{\sqrt{d_j +1}}  )^2 \geq 0.
    \end{align*}
}
Since identity matrix is also positive semi-definite,  we obtain the result.
\end{proof}

\subsection{GCNs with Composite Kernels}
Proposition~\ref{pro:A_hat_indefinite} shows that the aggregation coefficients in a GCN is an indefinite kernel, encoding only neighborhood information. This is limited in the sense that no node features are considered in the kernel. To overcome this limitation, we propose a composite kernel framework to encode both neighborhood and node-feature information. We first prove the composition kernel theory.
\begin{proposition}[Composite Kernel]\label{pro:product_of_indefinite_kernels}
Given two indefinite kernels $k_1$, $k_2$, function $k(\xb_i, \xb_j) = k_1(\xb_i, \xb_j)k_2(\xb_i, \xb_j)$ is also an indefinite kernel.
\end{proposition}
\begin{proof}
    {\small
        \begin{align*}
            &k(\xb_i, \xb_j)
            = k_1(\xb_i, \xb_j)k_2(\xb_i, \xb_j) \\
            =& k_{1+}(\xb_i, \xb_j)k_{2+}(\xb_i, \xb_j) + k_{1-}(\xb_i, \xb_j)k_{2-}(\xb_i, \xb_j) \\
            & - k_{1+}(\xb_i, \xb_j)k_{2-}(\xb_i, \xb_j) - k_{1-}(\xb_i, \xb_j)k_{2+}(\xb_i, \xb_j)~,
        \end{align*}
    }
where $k_{1+}$, $k_{1-}$ and $k_{2+}$, $k_{2-}$ are the decomposition of $k_1$ and $k_2$, respectively. 
By \citet{10.5555/975545}, $(k_{1+}k_{2+}+k_{1-}k_{2-})$ and $(k_{1+}k_{2-} + k_{1-}k_{2+})$ are positive semi-definite kernels, thus $k(\xb_i, \xb_j)$ is an indefinite kernel.
\end{proof}
Note since positive semi-definite kernels are special cases of indefinite kernels, if one of $k_1$ and $k_2$ is a positive semi-definite kernel, Proposition~\ref{pro:product_of_indefinite_kernels} still hol   ds.

Our proposed GCN with composite kernels builds on Proposition~\ref{pro:product_of_indefinite_kernels}. The idea is to design a kernel $k$ to encode feature similarities into the kernel matrix $\mathbf{K}$ ($\mathbf{K}$ is defined in section \ref{sec:kernel_learning} ). Combining $\mathbf{K}$ with $\hat{\mathbf{A}}$, we arrive at a composite kernel with matrix $\hat{\mathbf{K}} = \mathbf{K} \odot \hat{\mathbf{A}}$, where $\odot$ denotes the element-wise multiplication.

\begin{proposition}\label{pro:K_hat_indefinite}
    Matrix $\hat{\mathbf{K}} = \mathbf{K} \odot \hat{\mathbf{A}}$ is a valid kernel matrix of an indefinite kernel.
\end{proposition}
\begin{proof}
Follow immediately from Proposition \ref{pro:product_of_indefinite_kernels}.
\end{proof}

Notice that $\hat{\mathbf{A}}$ is sparse in most cases; thus matrix $\hat{\mathbf{K}}$ is also sparse, implying that efficient implementation is possible for Tensorflow or PyTorch 
in both computation and storage. By using similar notation as in GCN, the corresponding output of the $(l+1)^{th}$ layer in the proposed GCN with a composite kernel $\hat{\mathbf{K}} \triangleq \{\hat{k}_{i,j}\}$ becomes:
\[
\hb_i^{l+1} = \sigma(\sum _{j \in \mathcal{N}_i}\hat{k} _{i,j} \hb_j^l \mathbf{W}^l)~.
\]

To enhance the representation power, we use a structure similar to GAT to define the output of the $(l+1)^{th}$ layer as:
\begin{align}\label{CKGCN_1}
    \hb_i^{l+1} = \ReLU([\sum _{j \in \mathcal{N}_i}\hat{k} _{i,j} \hb_j^l \mathbf{W}^l \Vert \sum _{j \in \mathcal{N}_i} \hat{a} _{i,j} \hb_j^l \mathbf{W}^l ]).
\end{align}
Finally, the output layer is (i.e., $l = L-1$):
\begin{align}\label{CKGCN_2}
    \hb_i^{l+1} = \Softmax(\sum _{j \in \mathcal{N}_i}\hat{k} _{i,j} \hb_j^l \mathbf{W}^l + \sum _{j \in \mathcal{N}_i} \hat{a} _{i,j} \hb_j^l \mathbf{W}^l ).
\end{align}
We call the resulting model Graph Conovlutional Network with Composite Kernel, denoted as CKGCN.

In regard to computation complexity, since our kernel is parameterized by a neural network (we will discuss how to learn the kernel later) and $\hat{\mathbf{K}}$ is sparse in most cases, only $\mathcal{O}(\vert \mathcal{E}\vert)$ extra computation and space are needed, which is linear with respect to the number of edges.

\subsection{GATs with Composite Kernels}
We further extend our composite kernel idea to GAT. Note that \citet{tsai2019transformer} shows a relationship between kernel method and attention mechanism, where the unnormalized attention is treated as an asymmetric kernel, and the normalized attention \eqref{eq:att_computation} as a kernel smoothing factor. However, the paper does not study the positive definiteness of their kernel matrices. Instead of treating the normalized attention as a kernel smoothing factor, we show that it can be considered as an indefinite kernel, {\it i.e.}, it can be computed by two positive semi-definite kernels. To this end, we use $\mathbf{T}^{l,m}$ to denote the $m^{th}$ attention matrix in the $(l+1)^{th}$ layer, and first prove the following proposition.
\begin{proposition}
    The normalized attention matrix $\mathbf{T}^{l,m}$ can be computed by two positive semi-definite matrices. If the matrix is symmetric, it is equivalent to the kernel matrix of an indefinite kernel.
\end{proposition}
\begin{proof}
The proof follows the same steps as in Proposition \ref{pro:A_hat_indefinite}. It is easy to show that $\mathbf{I} - \mathbf{T}^{l,m}$ and $\mathbf{I}$ are positive semi-definite matrices.
\end{proof}
If, on the other hand, the normalized attention is asymmetric, we can treat it as indefinite asymmetric kernel, a natural extension of indefinite kernels. Techniques of asymmetric kernels can be found in \citet{wu2010asymmetric}. Specifically, we have the following proposition:
\begin{proposition}
    If $\mathbf{T}^{l,m}$ is asymmetric, $\mathbf{K}^{l,m} = \mathbf{K} \odot \mathbf{T}^{l,m}$ is a kernel matrix of an indefinite asymmetric kernel. If $\mathbf{T}^{l,m}$ is symmetric, $\mathbf{K}^{l,m} = \mathbf{K} \odot \mathbf{T}^{l,m}$ is a kernel matrix of an indefinite kernel.
\end{proposition}
\begin{proof}
It follows from Proposition \ref{pro:product_of_indefinite_kernels}
\end{proof}

It is now ready to define GAT with composite kernels. For a layer with a multi-head attention mechanism as in Section \ref{sec:pre}, we first compute $M$ independent kernel matrices $\mathbf{K}^{l,m} = \mathbf{K} \odot \mathbf{T}^{l,m}$ for the $(l+1)^{th}$ layer, where $m=1,2...M$. 
Denote the elements in $\mathbf{K}^{l,m}$ by $k^{l,m}_{i,j}$. We then define the output of the $(l+1)^{th}$ layer as:
\begin{align*}
    \hb_i^{l+1} = \Vert _{m=1}^M \ReLU([\sum _{j \in \mathcal{N}_i}k^{l,m} _{i,j} \hb_j^l \mathbf{W}^l \Vert \sum _{j \in \mathcal{N}_i} t^{l,m} _{i,j} \hb_j^l \mathbf{W}^l ]).
\end{align*}
For the output layer ($l = L - 1$), we define it as:
{\begin{align*}
    \hb_i^{l+1} = \Softmax(\sum_{m=1}^M &(\sum _{j \in \mathcal{N}_i}k^{l,m} _{i,j} \hb_j^{l} \mathbf{W}^{l} + \sum _{j \in \mathcal{N}_i} t^{l,m} _{i,j} \hb_j^{l} \mathbf{W}^{l}) ).
\end{align*}}
We call the resulting model Graph Attention Network with Composite Kernel, denote as CKGAT.

\subsection{Kernel Learning}\label{sec:kernel_learning}
The remaining challenge is to learn the kernel $k$ for both CKGCN and CKGAT. We will focus on positive semi-definite kernels since it is more convenient to encode node features. For simplicity and without loss of generality, we will also describe our solution in the context of node classification.

For node classification, one may want the following property to hold: nodes with the same label should be more similar than those with different labels. In the notation of kernel functions, this can be represented as: $k(\xb_i, \xb_j)\geq \alpha$ if $\yb_i=\yb_j$, and $k(\xb_i, \xb_j)\leq \beta$ otherwise, 
where $\alpha$ and $\beta$ are some threshold values. Consequently, learning the kernel can be formulated as the following optimization problem:
\begin{align}\label{eq:original_kernel_loss}
    \min_{\phib}  - \sum _{l=1}^c \mathbb{E}_{l_i = l_j = l}[k_{\phib}(\xb_i, \xb_j)] + \lambda \sum_{l_i \neq l_j}\mathbb{E}[k_{\phib}(\xb_i, \xb_j)]
\end{align}
where $\phib$ denotes the kernel parameters to be optimized, $c$ the class number, $l_i$ and $l_j$ the class of $\yb_i$ and $\yb_j$, respectively. 

We can further reformulate the above optimization problem as a feature distribution matching problem. Specifically, if we set $\lambda=1$, then \eqref{eq:original_kernel_loss} can be equivalently written as:
\begin{align} \label{eq: mmd_node}
    \min_{\phib} - \sum _{a=1}^c\sum_{b \neq a} \MMD_{k_{\phib}}^2(\mathbb{P}_a, \mathbb{P}_b),
\end{align}
where $\mathbb{P}_a$ denotes the node features distribution of $a^{th}$ class (note that there are $c$ classes in total). 

Because MMD is a distance metric between distributions, the above optimization problem can be interpreted as finding the kernel to best distinguish feature distributions of different class. Given the kernel parameterized by $\phib$, we want the MMD computed by two sets of nodes from different class to be large, whereas the MMD computed by two sets of nodes from the same class to be 0. For MMD, we first have the following result:
\begin{theorem}\citep{gretton2012kernel}\label{theorem:characteristic}
    Given a kernel k, if k is a characteristic kernel, then $\MMD_{k}(\mathbb{P}, \mathbb{Q})=0$ if and only if $\mathbb{P} = \mathbb{Q}$
\end{theorem}
\cite{sriperumbudur2010universality} shows that when the kernel is characteristic, MMD is a metric on the space of all the Radon probability measures. We now show how to construct a valid characteristic positive semi-definite kernel.

In our framework, we use an auto-encoder structure when representing a kernel. The auto-encoder consists of two components: an encoder function ($f_{enc}: R^d \rightarrow R^z$) and a decoder function ($f_{dec}: R^z \rightarrow R^d$), both are parameterized by neural networks. We use $\phib$ to denote the parameters in the auto-encoder. The encoder function maps an input sample from the original space into some point in the latent space (often has lower dimension than the original input space, {\it i.e.}, $z\leq d$), while the decoder function maps a latent variable back to the original space. We want the auto-encoder to be able to reconstruct the input data, and compress the original input data into a compact representation. The objective is defined as: 
{\small\begin{align}\label{eq:autoencoder}
    \min_{\phib} \sum _i\Vert \xb_i - \bar{\xb}_i\Vert ^2, \text{with }\bar{\xb}_i = f_{dec}(\zb_i), \zb_i = f_{enc}(\xb_i)~.
\end{align}}
Our proposed kernel has the following form:
\begin{align}\label{eq:kernel}
    k_{\phib}(\xb_i, \xb_j) = \exp(-\Vert f_{enc}(\xb_i) - f_{enc}(\xb_j)\Vert^2),
\end{align}
where $f_{enc}$ and $f_{dec}$ are obtained by optimizing \eqref{eq:autoencoder}. The resulting kernel is a positive semi-definite characteristic kernel according to the following proposition.
\begin{proposition}
    The kernel in \eqref{eq:kernel} is a positive semi-definite kernel, and it is characteristic in the ideal case.
\end{proposition}
\begin{proof}
    By \citet{10.5555/975545}, the kernel in \eqref{eq:kernel} is a positive semi-definite kernel.\\
    Because Gaussian kernel is characteristic \citep{sriperumbudur2010universality}, by \citet{gretton2012kernel}, a kernel $\Tilde{k}(\xb_i, \xb_j) = (k\circ f) (\xb_i, \xb_j) = k(f(\xb_i), f(\xb_j))$ is characteristic if $k$ is charateristic and $f$ is injective. In the ideal case of auto-encoder, $\Vert \xb_i - \bar{\xb}_i\Vert = 0$ for all $\xb_i$ in the data space, and $f_{enc}$ is the inverse function of $f_{dec}$. Because being bijective is a sufficient condition for being injective, and $f$ is bijective if and only if if it has an inverse function. Then  $f_{enc}$ is an injective function, thus the resulting kernel $\Tilde{k} = k \circ f_{enc}$ is a characteristic kernel. 
\end{proof}

Combining \eqref{eq: mmd_node} and \eqref{eq:autoencoder}, we can use the following loss function to learn the desired kernel function ($\beta$ is a hyper-parameter): 
{\small\begin{align}\label{eq: regularizer_kernel}
    \mathcal{L}_k &= \min_{\phib} \sum_i \Vert \xb_i - \bar{\xb}_i\Vert ^2  - \beta\sum _{a=1}^c\sum_{b \neq a} \MMD_{k_{\phib}}^2(\mathbb{P}_a, \mathbb{P}_b)~.
\end{align}}
In our implementation, to explicitly enforce the kernel value of nodes from the same class much larger than those from different class ({\it i.e.}, maximize their difference), we also add an additional regularizer: 
\begin{align} \label{eq:regularizer_difference}
    {\small
    \mathcal{L}_d = \min_{\phib}-\Vert \max_{i, j}(k_{\phib}(\xb_i, \xb_j)) - \min_{i, j}(k_{\phib}(\xb_i, \xb_j)) \Vert^2.
    }
\end{align}


Consequently, the final loss function is defined as:
\begin{align}\label{eq:final_loss}
    \mathcal{L} = \mathcal{L}_{ce} + \lambda_1 \mathcal{L}_{k} + \lambda_2 \mathcal{L}_{d},
\end{align}
where $\mathcal{L}_{ce}$ is the canonical cross-entropy loss used in classification tasks, and $\lambda_1$, $\lambda_2$ are hyper-parameters.  

\section{Experiments}\label{sec:exp}
We test our framework for node classification on different graph datasets, including citation networks and text documents.
\subsection{Citation Network Classification}
Cora (2708 nodes, 5429 edges), Citeseer (3327 nodes, 4732 edges), Pubmed (19717 nodes, 44338 edges) \citep{sen2008collective} are public citation network datasets. Nodes of a graph represent documents, and edges represent citation links in the graph. Documents are classified into several categories, hence each node is associated with a label. There are 7, 6, 3 categories in Cora, Citeseer, Pubmed respectively. The node feature dimensions are 1433, 3703, 500 respectively.

Following previous works of \citet{kipf2016semisupervised} and \cite{tian2019rethinking}, we test our proposed framework in both semi-supervised and supervised settings. Two-layer CKGCN and CKGAT are used in both experiments; and the Adam optimizer \citep{kingma2014adam} with $\beta_1=0.9$, $\beta_2=0.999$ is used. The encoder and decoder are set to be 4-layer neural networks for Cora an Pubmed, and a 5-layer neural network on Citeseer. The dimension of latent space are selected from 8 and 16 based on hyper-parameter tuning. Both dropout and L2 regularization are used.
\subsubsection{Semi-supervised Node Classification}
In semi-supervised experiments, we follow the setting in \citet{kipf2016semisupervised}: there are 20 nodes with labels for each class in the training set, 500 nodes in total for the validation set, and 1000 nodes in total for testing. 


The results are shown in Table \ref{tab:semi-super}. Results of some compared models are taken from the corresponding papers, without standard deviations. From the table we can see that the proposed composite kernel aggregation obtains best performance. Composite kernels in CKGCN and CKGAT also lead to better results than GCN and GAT, due to the flexibility of incorporating node features into the models. 
\begin{table}[tb!]
    \centering
    \scriptsize{
    \begin{tabular}{lrrr}
        \toprule
         Method&Cora&Citeseer&Pubmed  \\
         \midrule
         MLP&$55.1 $&$46.5 $&$71.4 $ \\
         GCN & $81.5 $& $70.3  $& $79.0  $\\
         ChebyNet & $81.2 $& $69.8  $& $74.4  $\\
         FastGCN & $79.8 \pm 0.3 $& $68.8 \pm0.6  $& $77.4 \pm 0.3  $\\
         AGNN & $83.1 \pm 0.1  $& $71.7 \pm 0.1  $& $79.9 \pm 0.0  $\\
         LNet & $79.5 \pm 1.8  $& $66.2 \pm 1.9  $& $78.3 \pm 0.3  $\\
         AdaLNet & $80.4 \pm 1.1  $& $68.7 \pm 1.0  $& $78.1 \pm 0.4  $\\
         DeepWalk & $70.7 \pm 0.6  $& $51.4 \pm 0.5  $& $76.8 \pm 0.6  $\\
         DGI & $82.3 \pm 0.6  $& $71.8 \pm 0.7  $& $76.8 \pm 0.6  $\\
         SGC & $81.0 \pm 0.0  $& $71.9 \pm 0.1  $& $78.9 \pm 0.0  $\\
         GWNN & 82.8  & 71.7   & 79.1  \\
         GAT & $83.0 \pm 0.7  $& $ 72.5 \pm 0.7  $& $79.0 \pm 0.3  $\\
         K3 & $78.5\pm 0.1  $& $68.4 \pm 0.1  $& $ 76.1 \pm 0.1  $\\
         \midrule
         CKGCN(ours) & $83.6\pm 0.1  $& $72.1 \pm 0.1  $& $\mathbf{80.3 \pm 0.1}  $\\
         CKGAT(ours) & $\mathbf{84.4\pm 0.1}  $& $\mathbf{74.0 \pm 0.1}  $& $80.2 \pm 0.1  $\\
         \bottomrule
    \end{tabular}
    }
    \caption{Test accuracy (\%) of semi-supervised node classification.}
    \label{tab:semi-super}
\end{table}

\subsubsection{Supervised Node Classification}
In supervised node classification experiments, we follow the data split settings in \citep{chen2018fastgcn}: for each dataset, there are 500 nodes in the validation set, 1000 nodes in the testing set, and the rest of nodes belong to the training set. In the comparison, both KLED \citep{fouss2012experimental} and K3 \citep{tian2019rethinking} are kernel related methods. Results are reported in Table~\ref{tab:super}. Similarly, our models achieve higher accuracies than other methods. 
To illustrate the difference between $\hat{\mathbf{K}}$ and $\hat{\mathbf{A}}$, we plot these two matrix as well as their differences ({\it i.e.}, $\Vert \hat{\mathbf{K}}-\hat{\mathbf{A}}\Vert$) on Cora in Figure \ref{fig:matrix_cora} after training. 

\begin{figure}[t!]
    \centering
        \subfigure[Kernel]{\includegraphics[width=0.29\linewidth]{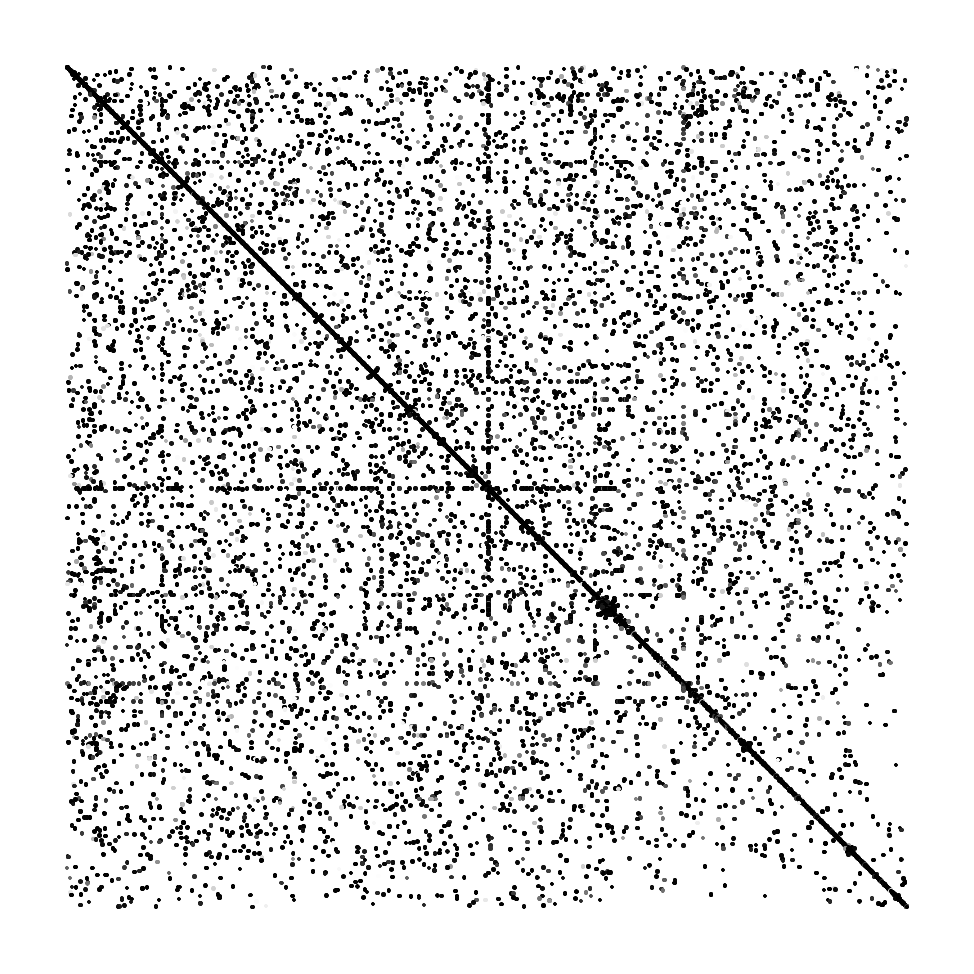}
        \label{fig:ker_matrix}}
        \subfigure[Adjacency]{\includegraphics[width=0.29\linewidth]{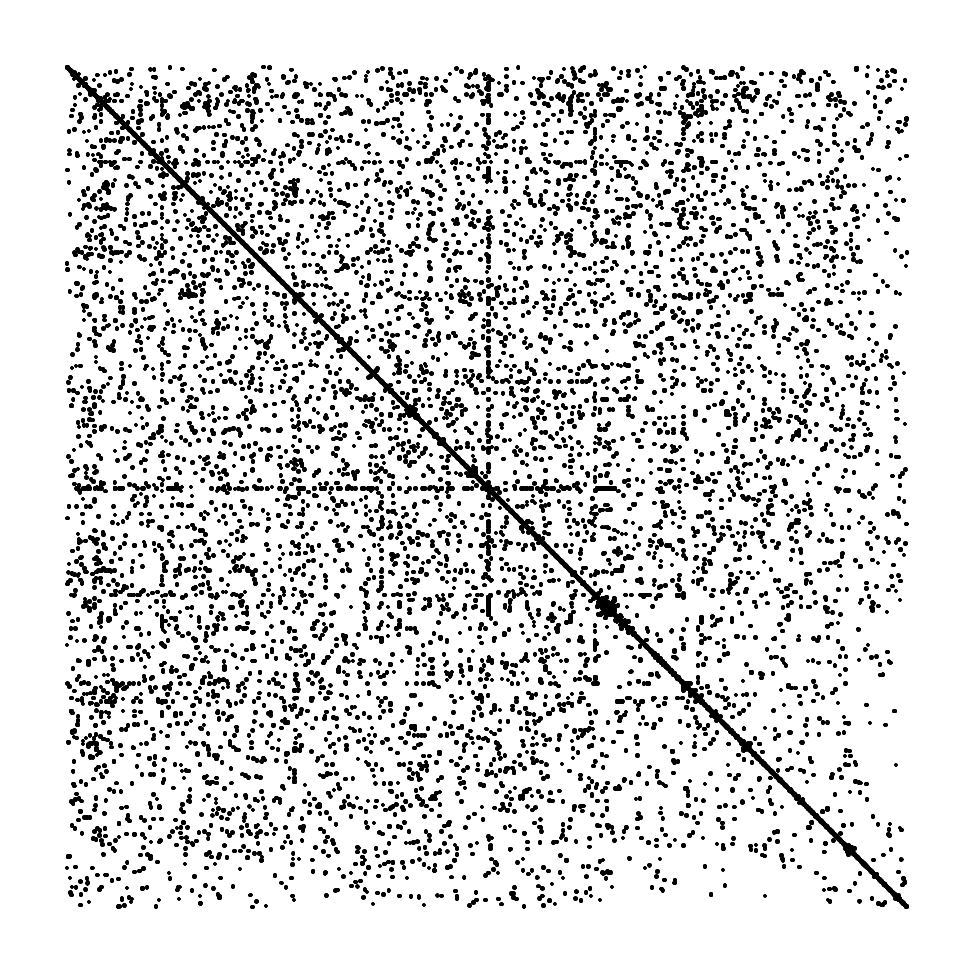}
        \label{fig:adj_matrix}}
        \subfigure[Difference]{\includegraphics[width=0.36\linewidth]{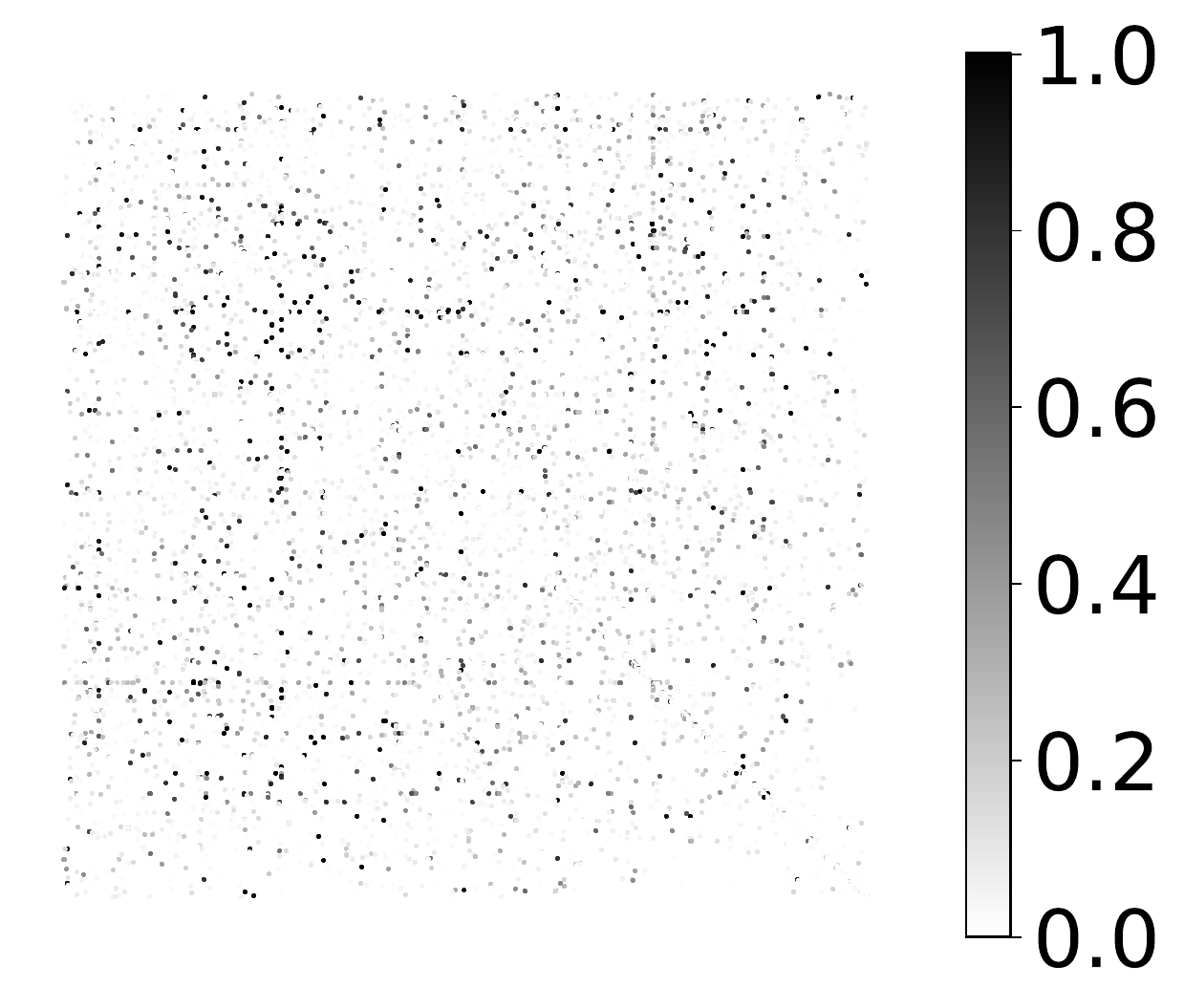}
        \label{fig:diff_matrix}}
    \caption{Comparison of $\hat{\mathbf{K}}$ and $\hat{\mathbf{A}}$ on Cora. \ref{fig:ker_matrix} shows values in $\hat{\mathbf{K}}$, \ref{fig:adj_matrix} shows values in $\hat{\mathbf{A}}$, \ref{fig:diff_matrix} shows the difference on every position between these two matrices.} 
    \label{fig:matrix_cora}
\end{figure}


\begin{table}[tb!]
    \centering
    \scriptsize{
    \begin{tabular}{lrrr}
        \toprule
         Method&Cora&Citeseer&Pubmed  \\
         \midrule
         GCN & $86.0 $& $77.2  $& $86.5 $\\
         FastGCN & $85.0  $& $77.6  $& $88.0  $\\
         GAT & $85.6  $& $76.9 $& 86.2\\
         KLED & $82.3  $& $- $& $82.3$\\
         K3 & $88.4\pm 0.2  $& $80.2 \pm 0.1  $& $89.4 \pm 0.1  $\\
         \midrule
         CKGCN(ours) & $\mathbf{88.9\pm 0.1}  $& $79.9 \pm 0.2  $& $\mathbf{90.0 \pm 0.1 }$\\
         CKGAT(ours) & $ 87.5 \pm 0.1 $& $\mathbf{80.3 \pm 0.1 }$& $ 87.3 \pm 0.1 $\\
         \bottomrule
    \end{tabular}
    }
    \caption{Test accuracy (\%) of supervised node classification.}
    \label{tab:super}
\end{table}
\subsection{Text Classification}
For text classification, we follow the setting in \citet{yao2018graph}, with dataset summaries given in Table \ref{tab:stats_text}. In each dataset, both documents and words are treated as nodes in a graph. Each node is associated with a 300-dimensional feature vector as in \citet{yao2018graph}. 
Following \citet{yao2018graph}, word-word edge weights are set to be point-wise mutual information, and word-document edges weights are set to be normalized TF-IDF scores. The proposed composite kernel method is applied as in previous section, which is denoted as CKGCN. 
We use a 2-layer CKGCN for every dataset. The encoder and decoder functions are both parameterized by 2-layer neural networks for all the datasets. The latent space dimensionality is set to be 16.
Dropout and Adam optimizer are used in the this experiment. 
Experimental results are summarized in Table \ref{tab:text_classification}, from which we can see that our CKGCN obtain the best results in the most of the datasets.
\begin{table}[t!]
    \centering
    \scriptsize{
        \begin{tabular}{lrrrrr}
        \toprule
         Datasets & Docs & Training / Test & Words & Nodes & Class   \\
         \midrule
          R8 & 7,674 & 5,485 / 2,189 & 7,688 & 15,362 & 8  \\
          R52 & 9,100 & 6,532 / 2,568 & 8,892 & 17,992 & 52 \\
          Ohsumed & 7,400 & 3,357 / 4,043 & 14,157 & 21, 557 & 23 \\
          MR & 10,662 & 7,108 / 3,554 & 18,764 & 29,426& 2\\
         \bottomrule
    \end{tabular}
    }
    \caption{Summary of datasets in text classification experiments.}
    \label{tab:stats_text}
\end{table}

\begin{table}[t!]
    \centering
    \scriptsize{
    \begin{tabular}{lrrrr}
        \toprule
         Method & R8 & R52 & Ohsumed & MR \\
         \midrule
        CNN-rand & $94.0 \pm 0.6$ & $85.4 \pm 0.5$ & $43.9 \pm 1.0$ & $75.0 \pm 0.7$ \\
        CNN-non-static & $95.7 \pm 0.5$ & $87.6 \pm 0.5$ & $58.4 \pm 1.1$ & $\mathbf{77.8 \pm 0.7}$ \\
        LSTM  & $93.7 \pm 1.5$ & $85.5 \pm 1.1$ & $41.1 \pm 1.2$ & $75.1 \pm 0.4$ \\
        LSTM (pretrain)  & $96.1 \pm 0.2$ &$90.5 \pm 0.9$ & $51.1 \pm 1.5$ & $77.3 \pm 0.9$ \\
        Bi-LSTM &  $96.3 \pm 0.3$ & $90.5 \pm 0.9$ & $49.3 \pm 1.1$ & $77.7 \pm 0.9$ \\
        ChebyNet & $97.0 \pm 0.1$ & $92.8 \pm 0.2$ & $63.9 \pm 0.5$ & $77.2 \pm 0.3$ \\
        GCN & $97.0 \pm 0.2$ & $93.8 \pm 0.2$ & $68.2 \pm 0.4$ & $76.3 \pm 0.3$ \\
        SGC & $97.1 \pm 0.1$ & $94.0 \pm 0.2$ & $68.5 \pm 0.3$ & $75.9 \pm 0.3$ \\
        Text-GCN & $97.1 \pm 0.1$ & $93.4 \pm 0.2$ & $68.4 \pm 0.6$ & $76.7 \pm 0.2$ \\
        \midrule
        CKGCN(ours) & $\mathbf{97.5 \pm 0.1}$ & $\mathbf{94.6 \pm 0.2}$ & $\mathbf{69.1 \pm 0.3}$ & $77.1 \pm 0.3$ \\
         \bottomrule
    \end{tabular}}
    \caption{Test accuracy (\%) of text classification.}
    \label{tab:text_classification}
\end{table}


\section{Conclusion}
In this paper, we propose a kernel-based framework for graph structural data, in node classification task specifically. We discuss the relationship between aggregation schemes and kernels and propose a composite kernel construction which leads to feature sensitive aggregation. How to parameterize the desired kernel by neural networks and how to learn the kernel are also discussed in the paper. We apply our framework in two major GNN models, GCN and GAT, and the resulting CKGCN and CKGAT achieve better experimental results than other models in different tasks.

\newpage
\bibliographystyle{named}
\bibliography{ijcai20}

\end{document}